\documentclass{article}

\usepackage{graphicx,subfigure}
\usepackage{epsfig}
\usepackage{hyperref}
\usepackage{amsmath}
\usepackage{amssymb}
\usepackage{algorithm}
\usepackage{algorithmic}
\usepackage{url}
\usepackage{enumerate}
\usepackage{amsfonts}
\usepackage{boxedminipage}
\usepackage{xcolor}
\usepackage{framed}  
\usepackage{rotating}
\usepackage{array}
\usepackage{multirow}
\usepackage{color}
\usepackage{tikz}
\usepackage{mathabx}
\usepackage{tabularx,ragged2e,booktabs}

\newcommand{\bE}{\mathbb{E}}
\newcommand{\reals}{\mathbb{R}}

\newcommand{\bN}{\mathbb{N}}

\newcommand{\poly}{\mathrm{poly}}

\newcommand{\cD}{\mathcal{D}}

\newcommand{\cB}{\mathcal{B}}
\newcommand{\cP}{\mathcal{P}}
\newcommand{\cA}{\mathcal{A}}

\newcommand{\cL}{\mathcal{L}}

\newcommand{\cW}{\mathcal{W}}

\newcommand{\cI}{\mathcal{I}}

\newenvironment{proof}{\par\noindent{\bf Proof\ }}{\hfill\BlackBox\\[2mm]}
\newcommand{\BlackBox}{\rule{1.5ex}{1.5ex}}  
\makeatletter
\def\moverlay{\mathpalette\mov@rlay}
\def\mov@rlay#1#2{\leavevmode\vtop{%
   \baselineskip\z@skip \lineskiplimit-\maxdimen
   \ialign{\hfil$\m@th#1##$\hfil\cr#2\crcr}}}
\newcommand{\charfusion}[3][\mathord]{
    #1{\ifx#1\mathop\vphantom{#2}\fi
        \mathpalette\mov@rlay{#2\cr#3}
      }
    \ifx#1\mathop\expandafter\displaylimits\fi}
\makeatother
\newcommand{\cupdot}{\charfusion[\mathbin]{\cup}{\cdot}}
\newcommand{\bigcupdot}{\charfusion[\mathop]{\bigcup}{\cdot}}

\newcommand{\regret}{{R}}

\newcommand{\indct}[1]{\boldsymbol{1}\!_{\left[ #1 \right]}}

\DeclareMathOperator*{\argmin}{arg\,min}
\DeclareMathOperator*{\argmax}{arg\,max} 
\renewcommand{\eqref}[1]{Equation~(\ref{#1})}

\newcommand{\figref}[1]{Figure~\ref{#1}}
\newcommand{\secref}[1]{Section~\ref{#1}}

\newcommand{\appref}[1]{Appendix~\ref{#1}}
\newcommand{\thmref}[1]{Theorem~\ref{#1}}
\newcommand{\lemref}[1]{Lemma~\ref{#1}}

\newcommand{\algref}[1]{Algorithm~\ref{#1}}

\usepackage{xr}
\newtheorem{theorem}{Theorem}
\newtheorem{lemma}{Lemma}

\newcommand{\act}[1]{\textrm{ACTIVE}(#1)}
\newcommand{\propref}[1]{Property~\ref{#1}}
\newcommand{\tregret}{\textrm{Tracking-Regret}}
\newcommand{\sareg}{\textrm{SA-Regret}}
\newcommand{\saol}{\textrm{SAOL}}
\begin{document} 

\title{Strongly Adaptive Online Learning}

% unless you've provided the [accepted] option to the icml2015
\author{Amit Daniely\footnote{Dept. of Mathematics, The Hebrew University, Jerusalem, Israel} \and Alon Gonen\footnote{School of Computer Science, The Hebrew University, Jerusalem, Isreal} \and Shai Shalev-Shwartz\footnote{School of Computer Science, The Hebrew University, Jerusalem, Isreal}}

%\keywords{adaptive learning, online learning}

%\vskip 0.3in
%]

\maketitle 

\begin{abstract}
\emph{Strongly adaptive algorithms} are algorithms whose performance on {\em every time interval} is close to optimal. 
We present a reduction that can transform standard low-regret algorithms to strongly adaptive. As a consequence, we derive simple, yet efficient, strongly adaptive algorithms for a handful of problems.
\end{abstract} 

\section{Introduction}  \label{sec:intro}
Coping with changing environments and rapidly adapting to changes is a key component in many tasks. A broker is highly rewarded from rapidly adjusting to new trends. A reliable routing algorithm must respond quickly to congestion. A web advertiser should adjust himself to new ads and to changes in the taste of its users. A politician can also benefit from quickly adjusting to changes in the public opinion. And the list goes on.

Most current algorithms and theoretical analysis focus on relatively stationary environments. In statistical learning, an algorithm should perform well on the training distribution. Even in online learning, an algorithm should usually compete with the best strategy (from a pool), that is fixed and does not change over time.

Our main focus is to investigate to which extent such algorithms can be modified to cope with changing environments. 

We consider a general online learning framework that encompasses various online learning problems including prediction with expert advice, online classification, online convex optimization and more. In this framework, a learning scenario is defined by a decision set $D$, a context space $C$ and a set $\cL$ of real-valued loss functions defined over $D$. The {\em learner} sequentially observes a context $c_t\in C$ and then picks a decision $x_t\in D$. Next, a loss function $\ell_t\in\cL$ is revealed and the learner suffers a loss $\ell_t(x_t)$.

Often, algorithms in such scenarios are evaluated by comparing their performance to the performance of the best strategy from a pool of strategies (usually, this pool is simply all strategies that play the same action all the time). Concretely, the {\em regret,} $R_{\cA}(T)$, of an algorithm $\cA$ is defined as its cumulative loss minus the cumulative loss of the best strategy in the pool. The rationale behind this evaluation metric is that one of the strategies in the pool is reasonably good during the {\em entire} course of the game. However, when the environment is changing, different strategies will be good in different periods. As we do not want to make any assumption on the duration of each of these periods, we would like to guarantee that our algorithm performs well on {\em every} interval $I=[q,s] \subset [T]$. Clearly, we cannot hope to have a regret bound which is better than what we have for algorithms that are tested only on $I$. If this barrier is met, we say that the corresponding algorithm is {\em strongly adaptive\footnote{See a precise definition in \secref{sec:setting}. Also, see \secref{sec:related} for a weaker notion of {\em adaptive algorithms} that was studied in \cite{hazan2007adaptive}.}}.

Surprisingly maybe, our main result shows that for many learning problems strongly adaptive algorithms exist. Concretely, we show a simple ``meta-algorithm" that can use any online algorithm (that was possibly designed to have just small standard regret) as a black box, and produces a new algorithm that is designed to have a small regret on every interval. We show that if the original algorithm have a regret bound of $R(T)$, then the produced algorithm has, on every interval $[q,s]$ of size $\tau:=|I|$, regret that is very close to $R(\tau)$ (see a precise statement in \secref{sec:results}). Moreover, the running time of the new algorithm at round $t$ is just $O\left(\log(t)\right)$ times larger than that of the original algorithm. As an immediate corollary we obtain strongly adaptive algorithms for a handful of online problems including prediction with expert advice, online convex optimization, and more.

Furthermore, we show that strong adaptivity is stronger than previously suggested adaptivity properties including the adaptivity notion of  \cite{hazan2007adaptive} and the tracking notion of \cite{herbster1998tracking}. Namely, strongly adaptive algorithms are also adaptive (in the sense of  \cite{hazan2007adaptive}), and have a near optimal tracking regret (in the sense of  \cite{herbster1998tracking}). We conclude our discussion by showing that strong adaptivity can not be achieved with bandit feedback.

\subsection{Problem setting}\label{sec:setting}
\subsubsection*{A Framework for Online Learning}  \label{sec:generalOnline}
Many learning problems can be described as a repeated game between the learner and the environment, which we describe below.

A {\em learning scenario} is determined by a triplet $(D, C, \cL)$, where $D$ is a \emph{decision space}, $C$ is a set of contexts, and $\cL$ is a set of \emph{loss} functions from $D$ to $[0,1]$. Extending the results to general bounded losses is straightforward. The number of rounds, denoted $T$, is unknown to the learner. At each time $t \in [T]$, the learner sees a context $c_t\in C$, and then chooses an action $x_t\in D$. Simultaneously, the environment chooses a loss function $\ell_t \in \cL$. Then, the action $x_t$ is revealed to the environment, and the loss function $\ell_t$ is revealed to the learner which suffers the loss $\ell_t(x_t)$. We list below some examples of families of learning scenarios.
\begin{itemize}
\item {\bf Learning with expert advice \cite{cesa1997use}.} Here, there is no context (formally, $C$ consists of a single element), $D$ is a finite set of size $N$ (each element in this set corresponds to an expert), and $\cL$ consists of all functions from $D$ to $[0,1]$.
\item {\bf Online convex optimization \cite{zinkevich2003online}.} Here, there is no context as well, $D$ is a convex set, and $\cL$ is a collection of convex functions from $D$ to $[0,1]$.
\item {\bf Classification.} Here, $C$ is some set, $D$ is a finite set, and $\cL$ consists of all functions from $D$ to $\{0,1\}$ that are indicators of a single element.
\item {\bf Regression.} Here, $C$ is a subset of a Euclidean space, $D=[0,1]$, and $\cL$ consists of all functions of the form $\ell(\hat y)=(y-\hat{y})^2$ for $y\in [0,1]$. 
\end{itemize}
A {\em learning problem} is a quadruple $\cP=(D, C, \cL, \cW)$, where $\cW$ is a benchmark of {\em strategies} that is used to evaluate the performance of algorithms. Here, each strategy $w \in \cW$ makes a prediction $x_t(w)\in \cD$ based on some rule. We assume that the prediction $x_t(w)$ of each strategy is fully determined by the game's history at the time of the prediction. I.e., by $(c_1,\ell_1),\ldots,(c_{t-1},\ell_{t-1}), c_t$. Usually, $\cW$ consists of very simple strategies. For example, in context-less scenarios (like learning with expert advice and online convex optimization), $\cW$ is often identified with $D$, and the strategy corresponding to $x\in D$ simply predicts $x$ at each step. In contextual problems (such as classification and regression), $\cW$ is often a collection of functions from $C$ to $D$ (a hypothesis class), and the prediction of the strategy corresponding to $h:C\to D$ at time $t$ is simply $h(c_t)$.

The cumulative loss of $w \in \cW$ at time $T$ is $L_w(T) = \sum_{t=1} ^T \ell_t(x_t(w))$ and the cumulative loss of an algorithm $\cA$ is $L_{\cA}(T) = \sum_{t=1} ^T \ell_t(x_t)$. The cumulative regret of $\cA$ is $\regret_\cA(T) = L_\cA(T) - \inf_{w \in \cW} L_w(T)$. We define the regret, $R_{\cP}(T)$, of the learning problem $\cP$ as the minimax regret bound. Namely, $R_{\cP}(T)$ is the minimal number for which there exists an algorithm $\cA$ such that for every environment $\regret_\cA(T)\le \regret_\cP(T)$. We say that an algorithm {\em $\cA$} has {\em low regret} if $R_{\cA}(T)= O\left(\poly\left(\log T\right)R_{\cP}(T)\right)$ for every environment.

We note that both the learner and the environment can make random decisions. In that case, the quantities defined above refer to the expected value of the corresponding terms.
\subsubsection*{Strongly Adaptive Regret}
Let $I = [q,s] := \{q,q+1,\ldots,s\} \subseteq [T]$. The loss of $w \in \cW$ during the interval $I$ is $L_{w}(I) = \sum_{t=q}^s \ell_t(x_t(w))$ and the loss of an algorithm $\cA$ during the interval $I$ is $L_\cA(I) = \sum_{t=q} ^s \ell_t(x_t)$. The regret of $\cA$ during the interval $I$ is $\regret_\cA (I) = L_\cA(I) - \inf_{w\in\cW}L_w(I)$. The {\em strongly adaptive regret} of $\cA$ at time $T$ is the function
\begin{equation*}
\sareg^T_{\cA} (\tau) = \max_{I=[q,q+\tau-1]\subset [T]}\regret_\cA (I) 
\end{equation*}
We say that $\cA$ is {\em strongly adaptive} if for every environment, $\sareg^T_{\cA}(\tau)= O\left(\poly\left(\log T\right)\cdot R_\cP(\tau)\right)$.

\subsection{Our Results}\label{sec:results}
\subsubsection*{A strongly adaptive meta-algorithm}
Achieving strongly adaptive regret seems more challenging than ensuring low regret. Nevertheless, we show that often, low-regret algorithms can be transformed into a strongly adaptive algorithms with a little extra computational cost.

Concretely, fix a learning scenario $(D,C,\cL)$. We derive a strongly adaptive meta-algorithm, that can use any algorithm $\cB$ (that presumably have low regret w.r.t. some learning problem) as a black-box.
We call our meta-algorithm Strongly Adaptive Online Learner (SAOL). The specific instantiation of SAOL that uses $\cB$ as the black box is denoted $\saol^{\cB}$. 

Fix a set $\cW$ of strategies and an algorithm $\cB$ whose regret w.r.t. $\cW$ satisfies
\begin{equation} \label{eq:blackbox}
\regret_{\cB}(T) \le C\cdot T^\alpha,
\end{equation}
where $\alpha \in (0,1)$, and $C > 0$ is some scalar. The properties of $\textrm{SAOL}^{\cB}$ are summarized in the theorem below. The description of the algorithm and the proof of \thmref{thm:SAOL} are given in \secref{sec:algorithm}. 
\begin{theorem}\label{thm:SAOL}
~
\begin{enumerate}
\item For every interval $I=[q,s] \subseteq \bN$,
\begin{equation*}
\regret_{\textrm{SAOL}^{\cB}}(I) \le  \frac{4}{2^{\alpha}-1} C |I|^{\alpha} + 40 \log(s+1) |I|^{\frac{1}{2}}~.
\end{equation*}
\item \label{cor:sa_exist} In particular, if $\alpha \ge \frac{1}{2}$ and $\cB$ has low regret, then $\textrm{SAOL}^{\cB}$ is strongly adaptive.
\item The runtime of SAOL at time $t$ is at most $\log(t+1)$ times the runtime per-iteration of $\cB$.
\end{enumerate}
\end{theorem}
From part \ref{cor:sa_exist}, we can derive strongly adaptive algorithms for many online problems. Two examples are outlined below.
\begin{itemize}
\item {\bf Prediction with $N$ experts advice.}
The Multiplicative Weights (MW) algorithm has regret $\le 2\sqrt{\ln(N)T}$. Hence, for every $I=[q,s] \subseteq [T]$,
\begin{align*}
\regret_{\textrm{SAOL}^{\textrm{MW}}}(I) = O\left( \left (\sqrt{\log(N)}+\log(s+1) \right ) \sqrt{|I|} \right)~.
\end{align*}
\item
{\bf Online convex optimization with $G$-Lipschitz loss functions over a convex set $D \subseteq \reals^d$ of diameter $B$.} Online Gradient Descent (OGD) has regret $\le 3BG\sqrt{T}$. Hence, for every $I=[q,s] \subseteq [T]$,
\begin{align*}
\regret_{\textrm{SAOL}^{\textrm{OGD}}}(I)= O \left ((BG+ \log(s+1)) \sqrt{|I|} \right )~.
\end{align*}
\end{itemize}

\subsubsection*{Comparison to (weak) adaptivity and tracking}
Several alternative measures for coping with changing environment were proposed in the literature. The two that are most related to our work are \emph{tracking regret} \cite{herbster1998tracking} and \emph{adaptive regret} \cite{hazan2007adaptive} (other notions are briefly discussed in \secref{sec:related}).

Adaptivity, as defined in \cite{hazan2007adaptive}, is a weaker requirement than strong adaptivity. The adaptive regret of a learner $\cA$ at time $T$ is $\max_{I \subseteq [T]} R_{\cA}(I)$. An algorithm is called {\em adaptive} if its adaptive regret is $O\left(\poly\left(\log T\right)R_{\cP}(T)\right)$.
For online convex optimization problems for which there exists an algorithm with regret bound $R(T)$, \cite{hazan2007adaptive} derived an efficient algorithm whose adaptive regret is at most $R(T) \log(T)+O\left(\sqrt{T \log^3(T)}\right)$, thus establishing adaptive algorithms for many online convex optimization problems. For the case where the loss functions are $\alpha$-exp concave, they showed an algorithm with adaptive regret $O(\frac{1}{\alpha} \log^2(T))$ (we note that according to our definition this algorithm is in fact strongly adaptive). A main difference between adaptivity and strong adaptivity, is that in many problems, adaptive algorithms are not guaranteed to perform well on small intervals. For example, for many problems including online convex optimization and learning with expert advice, the best possible adaptive regret is $\Omega(\sqrt{T})$. Such a bound is meaningless for intervals of size $O(\sqrt{T})$. We note that in many scenarios (e.g. routing, paging, news headlines promotion) it is highly desired to perform well even on very small intervals.

The problem of ``tracking the best expert'' was studied in \cite{herbster1998tracking} (see also, \cite{bousquet2003tracking}). In that problem, originally formulated for the learning with expert advice problem, learning algorithms are compared to all strategies that shift from one expert to another a bounded number of times. They derived an efficient algorithm, named Fixed-Share, which attains near-optimal regret bound of $\sqrt{Tm(\log(T)+\log(N))}$ versus the best strategy that shifts between $\le m$ experts. (Interestingly, a recent work \cite{cesa2012new} showed that the Fixed-Share algorithm is in fact (weakly) adaptive). 
As we show in \secref{sec:tracking}, strongly adaptive algorithms enjoy near-optimal tracking regret in the experts problem, and in fact, in many other problems (e.g., online convex optimization).
We note that as with (weakly) adaptive algorithms, algorithms with optimal tracking regret are not guaranteed to perform well on small intervals.

\subsubsection*{Strong adaptivity with bandit feedback}
In the so-called bandit setting, the loss functions $\ell_t$ is not exposed to the learner. Rather, the learner just gets to see the loss, $\ell_t(x_t)$, that he has suffered.
In \secref{sec:banditAdaptive} we prove that there are no strongly adaptive algorithms that can cope with bandit feedback. Even in the simple experts problem we show that for every $\epsilon>0$, there is no algorithm whose strongly adaptive regret is $O\left(|I|^{1-\epsilon}\cdot \poly\left(\log T\right)\right)$. Investigating possible alternative notions and/or weaker guarantees in the bandit setting is mostly left for future work.

\subsection{Related Work} \label{sec:related}
Maybe the most relevant previous work, from which we borrow many of our techniques is \cite{blum2007external}. 
They focused on the expert setting and proposed a strengthened notion of regret using time selection functions, which are functions from the time interval $[T]$ to $[0,1]$. The regret of a learner $\cA$ with respect to a time selection function $I$ is defined by $\regret_\cA^I(T) = \max_{i \in [N]} \left( \sum_{t=1} ^T I(t) \ell_t(x_t)- \sum_{t=1} ^T I(t) \ell_t(i) \right)$, where $\ell_t(i)$ is the loss of expert $i$ at time $t$. This setting can be viewed as a generalization of the sleeping expert setting \cite{freund1997using}. For a fixed set $\cI$ consisting of $M$ time selection functions, they proved a regret bound of $O(\sqrt{L_{\min,I} \log (NM))}+\log(NM))$ with\footnote{where $L_{min,I} = \min_i \sum_{t=1} ^T I(t) \ell_t(i)$} respect to each time selection function $I \in \cI$. We observe that if we let $\cI$ be the set of all indicator functions of intervals (note that $|\cI|=\binom{T}{2}=\Theta(T^2)$), we obtain a strongly adaptive algorithm for learning with expert advice. However, the (multiplicative) computational overhead of our algorithm (w.r.t. the standard MW algorithm) at time $t$ is $\Theta(\log(t))$, whereas the computational overhead of their algorithm is $\Theta(T^2)$. Furthermore, our setting is much more general than the expert setting. 

Another related, but somewhat orthogonal line of work \cite{zinkevich2003online,hall2013online,rakhlin2013optimization, jadbabaie2015online} studies {\em drifting environments}. The focus of those papers is on scenarios where the environment is changing slowly over time.

\section{Reducing Adaptive Regret to Standard Regret}\label{sec:algorithm}
In this section we present our strongly adaptive meta-algorithm, named \emph{Strongly Adaptive Online Learner} (SAOL). For the rest of this section we fix a learning scenario $(D,C,\cL)$ and an algorithm $\cB$ that operates in this scenario (think of $\cB$ as a low regret algorithm).

We first give a high level description of SAOL.
The basic idea is to run an instance of $\cB$ on each interval $I$ from an appropriately chosen set of intervals, denoted $\cI$. The instance corresponding to $I$ is denoted $\cB_I$, and can be thought as an expert that gives his advice for the best action at each time slot in $I$.
The algorithm weights the various $\cB_I$'s according to their performance in the past, in a way that instances with better performance get more weight. The exact weighting is a variant of the multiplicative weights rule. At each step, SAOL picks at random one of the $\cB_I$'s and follows his advice. The probability of choosing each $\cB_I$ is proportional to its weight. Next, we give more details.

{\bf The choice of $\cI$.} As in the MW algorithm, the weighting procedure is used to ensure that SAOL performs optimally for every $I\in \cI$. Therefore, the choice of $\cI$ exhibits the following tradeoff. On one hand, $\cI$ should be large, since we want that optimal performance on intervals in $\cI$ will result in an optimal performance on {\em every interval}. On the other hand, we would like to keep $\cI$ small, since running many instances of $\cB$ in parallel will result with a large computational cost. To balance these desires, we let 
\begin{equation*} 
\cI = \bigcup_{k \in \bN \cup \{0\}} \cI_k~,
\end{equation*}
where for all $k \in \bN \cup \{0\}$,
\begin{equation*} 
\cI_k = \{[i \cdot 2^k,(i+1) \cdot 2^k-1] ~:~ i \in \bN\}  .
%\cI_k = \bigcup_{i \in \bN} [i \cdot 2^k, (i+1) \cdot 2^k-1] 
\end{equation*}
That is, each $\cI_k$ is a partition of $\bN \setminus \{1,\ldots,2^k\}$ to consecutive intervals of length $2^k$. 
We denote by 
\[
\act{t} := \{I \in \cI \,:\, t \in I\}~,
\]
the set of active intervals at time $t$. By the definition of $\cI_k$, for every $t \le 2^k$ we have that no interval in $\cI_k$ contains $t$, while for every $t > 2^k$ we have that a single interval in $\cI_k$ contains $t$. Therefore, 
\[
|\act{t}| = \lfloor \log(t) \rfloor+1 ~.
\]
It follows that the running time of SAOL at time $t$ is at most $(\log(t)+1)$ times larger than the running time of $\cB$. On the other hand, as we show in the proof, we can cover every interval by intervals from $\cI$, in a way that will guarantee small regret on the covered interval, provided that we have small regret on the covering intervals.

{\bf The weighting method.} Let $x_t=x_t(I)$ be the action taken by $\cB_I$ at time $t$. The instantaneous regret of SAOL w.r.t. $\cB_I$ at time $t$ is $r_t(I) = \ell_t(x_t)-\ell_t(x_t(I))$. As explained above, SAOL maintains weights over the $\cB_I$'s. For $I=[q,s]$, the weight of $\cB_I$ at time $t$ is denoted $w_t(I)$. For $t<q$, $\cB_I$ is not active yet, so we let $w_t(I)=0$. At the ``entry'' time, $t=q$, we set $w_t(I)=\eta_I$ where
\[
\eta_I :=\min \left\{1/2,1/{\sqrt{|I|}} \right\}.
\] 
The weight at time $t \in (q,s]$ is the previous weight times 
$(1+\eta_{I} \cdot r_{t-1}(I))$.  Overall, we have
\begin{align}  \label{eq:weights}
w_t(I) = 
\begin{cases}
0 & t \notin I \\
\eta_I & t = q \\
w_{t-1}(I) (1+\eta_{I} \cdot r_{t-1}(I)) & t \in (q,s]~
\end{cases}
\end{align} 
Note that the regret is always between $[-1,1]$, and $\eta_I \in (0,1)$, therefore weights are always positive during the lifetime of the corresponding expert. Also, the weight of $B_I$ decreases (increases) if its loss is higher (lower) than the predicted loss. 

The overall weight at time $t$ is defined by 
\[W_t  := \sum_{I \in \cI} w_t(I)= \sum_{I \in \act{t}} w_t(I).
\] 
Finally, a probability distribution over the experts at time $t$ is defined by
\[
p_t(I) = \frac{w_t(I)}{W_t}~.
\]
Note that the probability mass assigned to any inactive instance is zero. The probability distribution $p_t$ determines the action of SAOL at time $t$. Namely, we have $x_t = x_t(I)$ with probability $p_t(I)$. A pseudo-code of SAOL is detailed in \algref{alg:SAOL}.
\begin{algorithm}
\caption{Strongly Adaptive Online Learner (with blackbox algorithm $\cB$)}
\label{alg:SAOL}
\begin{algorithmic}
\STATE Initialize: $w_1(I) = \begin{cases} 1/2 & I = [1,1] \\ 0 & \textrm{o.w.} \end{cases}$
\FOR{$t=1$ {\bfseries to} $T$} 
\STATE Let $W_t = \sum_{I \in \act{t}} w_t(I)$
\STATE Choose $I\in\act{t}$ w.p. $p_t(I)= \frac{w_t(I)}{W_t}$
\STATE Predict $x_t(I)$
\STATE Update weights according to \eqref{eq:weights}
\ENDFOR
\end{algorithmic}
\end{algorithm}

\subsection{Proof Sketch of \thmref{thm:SAOL}}
In this section we sketch the proof of \thmref{thm:SAOL}. A full proof is detailed in \appref{app:completeMain}. The analysis of SAOL is divided into two parts. The first challenge is to prove the theorem for the intervals in $\cI$ (see \lemref{lem:coveredIntervals}). Then, the theorem should be extended to any interval (end of \appref{app:completeMain}).

Let us start with the first task. Our first observation is that for every interval $I$, the regret of SAOL during the interval $I$ is equal to 
\begin{equation} \label{eq:relativeReg}
\textrm{(SAOL's regret relatively to $\cB_I$ $+$ the regret of $\cB_I$)}
\end{equation}
(during the interval $I$). Since the regret of $\cB_I$ during the interval $I$ is already guaranteed to be small (\eqref{eq:blackbox}), the problem of ensuring low regret during each of the intervals in $\cI$ is reduced to the problem of ensuring low regret with respect to each of the $\cB_I$'s.

We next prove that the regret of SAOL with respect to the $\cB_I$'s is small. Our analysis is similar to the proof of \cite{blum2007external}[Theorem 16]. Both of these proofs are similar to the analysis of the Multiplicative Weights Update (MW) method. The main idea is to define a potential function and relate it both to the loss of the learner and the loss of the best expert. 

To this end, we start by defining pseudo-weights over the experts (the $B_I$'s). With a slight abuse of notation, we define $I(t) = \indct{t \in I}$. For any $I=[q,s] \in \cI$, the pseudo-weight of $\cB_I$ is defined by:
\begin{align*}
\tilde{w}_t(I) = 
\begin{cases}
0 & t < q \\
1 & t = q \\
\tilde{w}_{t-1}(I) \cdot (1+\eta_{I}  \cdot r_{t-1}(I)) & q <t \le s+1\\
\tilde{w}_{s}(I) & t > s+1
\end{cases}
\end{align*} 
Note that
\begin{equation*}
w_t(I) = \eta_I \cdot I(t) \cdot \tilde{w}_t(I)~.
\end{equation*}
The potential function we consider is the overall pseudo-weight at time $t$, $\tilde{W}_t = \sum_{I \in \cI} \tilde{w}_t(I)$. The following lemma, whose proof is given in the appendix, is a useful consequence of our definitions.
\begin{lemma} \label{lem:upperPotential}
For every $t \ge 1$,
\[
\tilde{W}_t \le  t (\log (t)+1)~.
\]
\end{lemma}
Through straightforward calculations, we conclude the proof of \thmref{thm:SAOL} for any interval in $\cI$. 
\begin{lemma} \label{lem:coveredIntervals}
For every $I=[q,s] \in \cI$, 
\[
\sum_{t=q} ^s r_t(I) \le  5 \log(s+1) \sqrt{|I|}~.
\]
Hence, according to \eqref{eq:relativeReg},
\begin{align*}
\regret_{\textrm{SAOL}^{\cB}}(I) &\le C \cdot |I|^\alpha + 5 \log(s+1) \sqrt{|I|}  \\&
\end{align*}
\end{lemma}
The proof is given in the appendix.

The extension of the theorem to any interval relies on some useful properties of the set $\cI$ (see \lemref{lem:intervalFamilyProp} in the appendix). Roughly speaking, any interval $I \subseteq [T]$ can be partitioned into two sequences of intervals from $\cI$, such that the lengths of the intervals in each sequence decay at an exponential rate (\lemref{lem:intervalPartition} in the appendix). The theorem now follows by bounding the regret during the interval $I$ by the sum of the regrets during the intervals in the above two sequences, and by using the fact that the lengths decay exponentially.
\section{Strongly Adaptive Regret Is Stronger Than Tracking Regret} \label{sec:tracking}
In this section we relate the notion of strong adaptivity to that of tracking regret, and show that algorithms with small strongly adaptive regret also have small tracking regret. Let us briefly review the problem of tracking. For simplicity, we focus on context-less learning problems, and on the case where the set of strategies coincides with the decision space (though the result can be straightforwardly generalized).
Fix a decision space $D$ and a family $\cL$ of loss functions. A compound action is a sequence $\sigma = (\sigma_1, \ldots, \sigma_T) \in D^T$. Since there is no hope in competing w.r.t. all sequences\footnote{It is easy to prove a lower bound of order $T$ for this problem}, a typical restriction of the problem is to bound the number of switches in each sequence. For a positive integer $m$, the class of compound actions with at most $m$ switches is defined by
\begin{equation} \label{eq:compound}
B_m = \left\{\sigma \in D^T: s(\sigma):=\sum_{t=1}^{T-1} \indct {\sigma_{t+1} \neq \sigma_{t}} \le m \right\}~.
\end{equation}
The notions of loss and regret naturally extend to this setting. For example, the cumulative loss of a compound action $\sigma \in B_m$ is defined by $L_{\sigma}(T)= \sum_{t=1} ^T \ell_t(\sigma_t)$. The tracking regret of an algorithm $\cA$ w.r.t. the class $B_m$ is defined by
$$
\tregret_{\cA}^m (T) = L_\cA(T) - \inf_{\sigma \in B_m} L_\sigma(T)~.
$$
The following theorem bounds the tracking regret of algorithms with bounds on the strongly adaptive regret. In particular, of SAOL.
\begin{theorem}  \label{thm:trackingRegretSAOL}
Let $\cA$ be a learning algorithm with $\sareg_{\cA}(\tau)\le C \tau^{\alpha}$. Then,
\[
\tregret_{\cA}^m(T) \le CT^{\alpha}m^{1-\alpha}
\]
\end{theorem}
\begin{proof}
Let $\sigma \in B_m$. Let $I_1,\ldots, I_m$ be the intervals that correspond to $\sigma$. Clearly, the tracking regret w.r.t. $\sigma$ is bounded by the sum of the regrets of during the intervals $I_1, \ldots, I_m$. Hence, and using H\"{o}lder's inequality, we have
\begin{align*}
L_{\cA}(T) - L_\sigma(T) &\le \sum_{i=1} ^m \regret_{\cA}(I_i)
\\
&\le  C\sum_{i=1} ^m |I_i|^{\alpha}
\\
&\le  C\left(\sum_{i=1}^m1^{\frac{1}{1-\alpha}}\right)^{1-\alpha} \left(\sum_{i=1}^m|I_i|\right)^{\alpha}
\\
&\le  Cm^{1-\alpha} T^{\alpha}
\end{align*}
\end{proof}
Recall that for the problem of prediction with expert advice, the strongly adaptive regret of SAOL (with, say, Multiplicative Weights as a black box) is $O\left((\sqrt{\ln(N)}+\log(T))\sqrt{\tau}\right)$. Hence, we obtain a tracking bound of $O\left((\sqrt{\ln(N)}+\log(T))\sqrt{mT}\right)$.
Up to a $\sqrt{\log(T)}$ factor, this bound is asymptotically equivalent to the bound of the Fixed-Share Algorithm of \cite{herbster1998tracking}\footnote{For the comparison, we rely on a simplified form of the bound of the Fixed-Share algorithm. This simplified form can be found, for example, in \url{http://web.eecs.umich.edu/~jabernet/eecs598course/web/notes/lec5_091813.pdf}}. Also, up to $\log(T)$ factor, the bound is optimal. One advantage of SAOL over Fixed-Share is that SAOL is parameter-free. In particular, SAOL does not need to know\footnote{The parameters of Fixed-Share do depend on $m$} $m$.

\section{Strongly Adaptive Regret in The Bandit Setting}  \label{sec:banditAdaptive}
In this section we consider the challenge of achieving adaptivity in the bandit setting. Following our notation, in the bandit setting, only the loss incured by the learner, $\ell_t(x_t)$, is revealed at the end of each round (rather than the loss function, $\ell_t$). For many online learning problems for which there exists an efficient low-regret algorithm in the full information model, a simple reduction from the bandit setting to the full information setting (for example, see \cite{shalev2011online}[Theorem 4.1])  yields an efficient low-regret bandit algorithm. Furthermore, it is often the case that the dependence of the regret on $T$ is not affected by the lack of information. For example, for the Multi-armed bandit (MAB) problem \cite{auer2002nonstochastic} (which is the bandit version of the the problem of prediction with expert advice), the above reduction yields an algorithm with near optimal regret bound of $2\sqrt{T N\log N }$.

A natural question is whether adaptivity can be achieved with bandit feedback. Few positive results are known. For example, applying the aforementioned reduction to the Fixed-Share algorithm results with an efficient bandit learner whose tracking regret is $O\left(\sqrt{Tm(\ln(N)+\ln(T))N}\right)$.

The next theorem shows that with bandit feedback there are no algorithms with non-trivial bounds on the strongly adaptive regret. We focus on the MAB problem with two arms (experts) but it is easy to generalize the result to any nondegenerate online problem. Recall that for this problem we do not have a context, $\cW=D=\{e_1,e_2\}$ and $\cL=[0,1]^{\cD}$. 
\begin{theorem}\label{thm:noStrongAdaptiveBandit}
For all $\epsilon>0$, there is no algorithm for MAB with strongly adaptive regret of $O\left(\tau^{1-\epsilon}\poly\left(\log T\right)\right)$.
\end{theorem}
The idea of the proof is simple. Suppose toward a contradiction that $\cA$ is an algorithm with strongly adaptive regret of $O\left(\tau^{1-\epsilon}\poly\left(\log T\right)\right)$.
This means that the regret of $\cA$ on every interval $I$ of length $T^{\frac{\epsilon}{2}}$ is non 
trivial (i.e. $o(|I|)$). Intuitively, this means that both arms must be inspected at least once during $I$. Suppose now that one of the arms is always superior to the second (say, has loss zero while the other has 
loss one). By the above argument, the algorithm will still inspect the bad arm at least once in every $T^{\frac{\epsilon}{2}}$ time slots. Those inspections will result in a regret of $\frac{T}{T^{\frac{\epsilon}{2}}}=T^{1-\frac{\epsilon}{2}}$. This, however, is a contradiction, since the strongly adaptive regret bound implies that the standard regret of $\cA$ is $o\left(T^{1-\frac{\epsilon}{2}}\right)$. 

This idea is formalized in the following lemma. It implies 
\thmref{thm:noStrongAdaptiveBandit}	as for $\cA$ with strongly adaptive 
regret of $O\left(\tau^{1-\epsilon}\poly\left(\log T\right)\right)$ we 
can take $k=O\left(T^{1-\frac{\epsilon}{2}}\right)$ and reach a 
contradiction as the lemma implies that on some segment $I$ of size $\frac{T}{k}=\Omega\left(T^{\frac{\epsilon}{2}}\right)$,
the regret of $\cA$ is $\Omega\left(T^{\frac{\epsilon}{2}}\right)$ which grows faster than $|I|^{1-\epsilon}\poly(\log T)$

\begin{lemma} 
Let $\cA$ be an algorithm with regret bounded
\[
\regret_\cA(T) \le k=k(T)~,
\]
Then, there exists an interval $I \subseteq [T]$ of size $\Omega (T/k)$ with
\[
\regret_\cA(I) = \Omega(|I|)~.
\]
\end{lemma}

\begin{proof}
Assume for simplicity that $4k$ divides $T$. Consider the environment $E^0$ , in which $\forall t,\;\ell_t(e_1)=0.5,\ell_t(e_2)=1$. 
Let $U\subset [T]$ be the (possibly random) set of time slots in which the algorithm chooses $e_2$ when the environment is $E^0$. Since the regret is at most $k$, we have $\bE[|U|]\le 2k$. It follows that for some segment $I\subset [T]$ of size $\ge \frac{T}{4k}$ we have $\bE[|U\cap I|]\le \frac{1}{2}$. Indeed, otherwise, if $[T] = I_1 \cupdot \ldots \cupdot I_{4k}$ is the partition of the interval $[T]$ into $4k$ disjoint and consecutive intervals of size $\frac{T}{4k}$ we will have $\bE[|U|]=\sum_{j=1}^{4k}~\bE[|U\cap I_j|]>2k$.

Now, since $|U\cap I|$ is a non-negative integer, w.p. $\ge \frac{1}{2}$ we have $|U\cap I|=0$. 
Namely, w.p. $\ge \frac{1}{2}$ $\cA$ does not inspect $e_2$ during the interval $I$ when it runs against $E^0$.
Consider now the environment $E$ that is identical to $E^0$, besides that $\forall t\in I,\; l_t(e_2)=0$. By the argument above, w.p. $\ge \frac{1}{2}$, the operation of $\cA$ on $E$ is identical to its operation on $E^0$. In particular, the regret on $I$ when $\cA$ plays against $E$ is, w.p. $\ge \frac{1}{2}$,  $\frac{|I|}{2}$, and in total, $\ge \frac{1}{2}\cdot\frac{1}{2}\cdot |I|$.
\end{proof}

\section*{Acknowledgments} 
We thank Yishay Mansour and Sergiu Hart for helpful discussions.
This work is supported by the Intel Collaborative Research Institute for Computational Intelligence (ICRI-CI). A. Daniely is supported by the Google Europe Fellowship in Learning Theory.

\newpage 

\bibliographystyle{plain}
\bibliography{bib}

\newpage
\appendix

\section{Proof of \thmref{thm:SAOL}}  \label{app:completeMain}

\subsection{Proving \thmref{thm:SAOL} to Any Interval in $\cI$}
\begin{proof} \textbf{(of \lemref{lem:upperPotential})}
The proof is by induction on $t$. For $t=1$, we have
\[
\tilde{W}_1 = \tilde{w}_1([1,1]) = 1~.
\]
Next, we assume that the claim holds for any $t' \le t$ and prove it for $\,t+1$. Since $|\{[q,s] \in \cI: q = t\}| \le \lfloor \log(t) \rfloor + 1$ for all $t \ge 1$, we have
\begin{align*}
\tilde{W}_{t+1} &= \sum_{I=[q,s] \in \cI} \tilde{w}_{t+1}(I) \\&
= \sum_{I=[t+1,s] \in \cI} \tilde{w}_{t+1}(I)+ \sum_{\substack{I=[q,s] \in \cI:\\ q \le t}} \tilde{w}_{t+1}(I) \\&
\le \log(t+1)+1 + \sum_{\substack{I=[q,s] \in \cI:\\ q \le t}} \tilde{w}_{t+1}(I)~.
\end{align*}
Next, according to the induction hypothesis, we have
\begin{align*}
\sum_{\substack{I=[q,s] \in \cI:\\ q \le t}} \tilde{w}_{t+1}(I) &= \sum_{\substack{I=[q,s] \in \cI:\\ q \le t}} \tilde{w}_t(I)(1+\eta_I \cdot I(t) \cdot r_t(I)) \\&
=  \tilde{W}_t + \sum_{I \in \cI} \eta_I \cdot I(t) \cdot r_t(I) \cdot \tilde{w}_t(I) \\& 
\le t (\log (t) +1)+  \sum_{I \in \cI} w_t(I) \cdot r_t(I)~.
\end{align*}
Hence,
\begin{align*}
\tilde{W}_{t+1} &\le t(\log(t)+1) + \log(t+1) + 1 + \sum_{I \in \cI} w_t(I) \cdot r_t(I) \\&
\le (t+1) (\log(t+1)+1) + \sum_{I \in \cI} w_t(I) \cdot r_t(I)~.
\end{align*}
We complete the proof by showing that $\sum_{I \in \cI} w_t(I) \cdot r_t(I)=0$. Since $x_t = x_{I,t}$ with probability $p_t(I)$ for every $I \in \cI$, we obtain
\begin{align*}
\sum_{I \in \cI} w_t(I) \cdot r_I(t) &=  W_t \sum_{I \in \cI} p_t(I) (\ell_t(x_t)-\ell_t(x_t(I))) \\&
= W_t (\ell_t(x_t) - \ell_t(x_t)) \\&
= 0~.
\end{align*}
Combining the above inequalities, we conclude the lemma.
\end{proof}

\begin{proof} \textbf{(of \lemref{lem:coveredIntervals})}
Fix some $I=[q,s] \in \cI$. We need to show that
\[
\sum_{t=q} ^s r_t(I) \le  5 \log(s+1) \sqrt{|I|} ~.
\]
Since weights are non-negative, using \lemref{lem:upperPotential}, we obtain
$$
\tilde{w}_{s+1}(I) \le \tilde{W}_{s+1} \le (s+1)(\log(s+1)+1) ~,
$$
Hence, 
\begin{equation} \label{eq:upperCovered}
\ln(\tilde{w}_{s+1}(I)) \le  \ln (s+1) + \ln (\log(s+1)+1)~.
\end{equation}
Next, we note that
\begin{align*}
\tilde{w}_{s+1}(I) = \prod_{t=q} ^s (1+\eta_I \cdot I(t) \cdot r_t(I)) = \prod_{t=q} ^s (1+\eta_I \cdot   r_t(I)) ~.
\end{align*}
Noting that $\eta_I \in (0,1/2)$ and using the inequality $\ln(1+x) \ge x - x^2$ which holds for every $x \ge -1/2$, we obtain
\begin{align} \label{eq:lowerCovered}
\ln(\tilde{w}_{s+1}(I))&= \sum_{t=q} ^ s \ln(1+ \eta_I \cdot  r_t(I) ) \notag \\&
\ge \sum_{t=q} ^s \eta_I  \cdot r_t(I) - \sum_{t=q} ^s (\eta_I \cdot  r_t(I))^2 \notag \\& 
\ge \eta_I (\sum_{t=q} ^s r_t(I) - \eta_I |I|) ~.
\end{align}
Combining \eqref{eq:lowerCovered} and \eqref{eq:upperCovered} and dividing by $\eta_I$, we obtain
\begin{align*}
\sum_{t=q} ^s r_t(I) &\le \eta_I |I|+ \eta_I^{-1} (\ln(s+1) + \ln(\log(s+1)+1)) \\&
\le \eta_I |I| + \eta_I^{-1}(\log(s+1) + \log(s+1)) \\&
\le \eta_I |I| + 2 \eta_I^{-1} \log(s+1) ~,
\end{align*}
where the second inequality follows from the inequality $x \ge \ln(1+x)$.
Substituting $\eta_I:=\min \left\{1/2,\frac{1}{\sqrt{|I|}} \right\}$, we conclude the lemma.
\end{proof}

\subsection{Extending The Theorem to Any Interval} 
In the next part we complete the proof of \thmref{thm:SAOL} by extending \lemref{lem:coveredIntervals} to every interval. 

Before proceeding, we set up an additional notation and also make some simple but useful observations regarding the properties of the set $\cI$ (defined in \secref{sec:algorithm}).

For an interval $J \subseteq \bN$, we define the restriction of $\cI$ to $J$ by $\cI|_J$. That is, $\cI|_J =\{I \in \cI: I \subseteq J\}$. We next list some useful properties of the set $\cI$ that follow immediately from its definition (thus, we do not prove these claims).
\begin{lemma} \label{lem:intervalFamilyProp}  \mbox{}
\begin{enumerate}
\item \label{prop:powerOfTwo}
The size of every interval $I \in \cI$ is $2^j$ for some $j \in \bN \cup \{0\}$.
\item \label{prop:leftmostInterval}
For every $j \in \bN \cup \{0\}$, the left endpoint of the leftmost interval $I$ whose size is $2^j$ is $2^j$. Thus, the size of every interval which is located to the left of $I$ is smaller than $|I|=2^j$. 
\item \label{prop:leftExist}
Let $I=[q,s] \in \cI$ be an interval and let $I'=[q',q-1]$ be another interval of size $2^{j} |I|$ for some $j \le 0$. Then, $I' \in \cI$.
\item \label{prop:increasePower}
Let $I=[q,s] \in \cI$ be an interval and let $I'=[s+1,s']$ be a consecutive interval of size $2^j |I|$ for some $j \le 0$. Then, $I' \in \cI$.
\item \label{prop:increasePowerAmb}
Let $I=[q,s] \in \cI$ be an interval of size $2^j$ for some $j \in \bN \cup \{0\}$. Then, (exactly) one of the intervals $[q,q+2^{j+1}-1]$, $[s+1,s+2^{j+1}]$ (whose size is $2^{j+1}$) belongs to $\cI$.
\end{enumerate}
\end{lemma}

The following lemma is a key tool for extending \lemref{lem:coveredIntervals} to any interval.
\begin{lemma} \label{lem:intervalPartition}
Let $I=[q,s] \subseteq \bN$ be an arbitrary interval. Then, the interval $I$ can be paritioned into two finite sequences of disjoint and consecutive intervals, denoted $(I_{-k},\ldots,I_{0}) \subseteq \cI|_I$ and $(I_1,I_2,\ldots,I_p) \subseteq \cI|_I$, such that
\[
(\forall i \ge 1)~~~~~|I_{-i}|/|I_{-i+1}|  \le 1/2~.
\]
\[
(\forall i \ge 2)~~~~~|I_i|/|I_{i-1}| \le 1/2~.
\]
\end{lemma}
The lemma is illustrated in \figref{fig:cover}. We next prove the lemma. Whenever we mention Property $1,\ldots,5$, we refer to Property $1,\ldots,5$ of \lemref{lem:intervalFamilyProp}.
\begin{proof} 
Let $b_0 = \max \{|I'|: I' \in \cI|_I\}$ be the maximal size of any interval $I' \in \cI$ that is contained in $I$. Among all of these intervals, let $I_0$ be the leftmost interval, i.e., we define 
\begin{align*}
&q_0:= \argmin \{q':[q',q'+b_0-1] \in \cI|_I\} ~\\&
s_0 = q_0+b_0-1~\\&
I_o = [q_0,s_0]~.
\end{align*}
Starting from $q_0-1$, we define a sequence of disjoint and consecutive intervals (in a reversed order), denoted $(I_{-1},\ldots,I_{-k})$, as follows:
\begin{align*}
[q_{-1},s_{-1}]&:=I_{-1}\\
&:= \argmax_{\substack{I'=[q',s'] \in \cI|_{[q,q_0-1]}:\\s'=q_0-1 }}  |I'| \\
\vdots \\
[q_{-i},s_{-i}]&:=I_{-i}\\
&:= \argmax_{\substack{I'=[q',s'] \in \cI|_{[q,q_{-i+1}-1]}: \\ s'=q_{-i+1} -1}} |I'| \\
\vdots
\end{align*}
Clearly, this sequence is finite and the left endpoint of the leftmost interval, $I_{-k}$, is $q$. Denote the size of $I_{-i}$ by $b_{-i}$. We next prove that for every $i \ge 1$, $b_{-i} / b_{-i+1} = 2^{j}$ for some $j \le -1$. We note that according to \propref{prop:powerOfTwo}, it suffices to show that $b_{-i}< b_{-i+1}$ for every $i \ge 1$. We use induction. The base case follows from the minimality of $I_0$. We next assume that the claim holds for every $i \in \{1,\ldots,k-1\}$ and prove for $k$. Assume by contradiction that $b_{-k} \ge b_{-k+1}$.  Consider the interval $\hat{I}_{-k+1}$ which is obtained by concatenating a copy of $I_{-k+1}$ to its left\footnote{Formally, $\hat{I}_{-k+1} :=[q_{-k+1}-b_{-k+1},q_{-k+1}-1] \cup I_{-k+1}$.}. It follows that $\hat{I}_{-k+1}$ is an interval of size $2 b_{-k+1}$ which is contained in $[q,q_{-k+2}-1]$ and its right endpoint is $q_{-k+2}-1$. According to the induction hypothesis, $|\hat{I}_{-k+1}| = 2 b_{-k+1} = 2^{j} \cdot b_{-k+2}$ for some $j \le 0$. It follows from \propref{prop:leftExist} that $\hat{I}_{-k+1} \in \cI|_I$, contradicting the maximality of $I_{-k+1}$. 

Similarly, starting from $s_0+1$, we define a sequence of disjoint and consecutive intervals, denoted $(I_1,\ldots,I_p)$:
\begin{align*}
  [q_1,s_1]&:=I_1\\
&:= \argmax_{\substack{I'=[q',s'] \in \cI|_{[s_0+1,s]}:\\q'=s_0+1 }}  |I'| \\
\vdots \\
[q_i,s_i]&:=I_i\\
&:= \argmax_{\substack{I'=[q',s'] \in \cI|_{[s_{i-1}+1 ,s ]}: \\ q'=s_{i-1} +1}} |I'| \\
\vdots
\end{align*}
Clearly, this sequence is finite and the right endpoint of the rightmost interval, $I_p$, is $s$. Denote the size of $I_i$ by $b_i$.  We next prove that for every $i \ge 2$, $b_{i} / b_{i-1} = 2^{j}$ for some $j \le -1$. According to \propref{prop:powerOfTwo}, it suffices to prove that $b_i < b_{i-1}$ for every $i \ge 2$. For this purpose, we first note that $b_1 \le b_0$; this follows immediately from the definition of $b_0$. Hence, we may assume that $b_i/b_{i-1} \in \{2^{j}: j \le 0\}$ for every $i \in \{1,\ldots,p-1\}$ and prove that $b_p < b_{p-1}$. Assume by contradiction that $b_p \ge b_{p-1}$. Consider the interval $\hat{I}_{p-1}$ which is obtained by concatenating a copy of $I_{p-1}$ to its right. It follows that $\hat{I}_{p-1}$ is an interval of size $2 b_{p-1}$ which is contained in $[s_{p-2}+1,s]$ and its left endpoint is $s_{p-2}+1$. According to the induction hypothesis, $|\hat{I}_{p-1}| = 2 b_{p-1} = 2^j \cdot b_{p-2}$ for some $j \le 1$. We need to consider the following two cases:
\begin{itemize}
\item
Assume first that $j \le 0$ (thus, $b_{p-1}/b_{p-2} \le 1/2$). Then, it follows from \propref{prop:increasePower} that $\hat{I}_{p-1} \in \cI|_I$, contradicting the maximality of $I_{p-1}$. 
\item
Assume that $j=1$ (i.e., $b_{p-1} = b_{p-2}$). Then, using \propref{prop:increasePowerAmb}, we obtain a contradiction to the maximality of $I_{k-2}$.
\end{itemize}
\end{proof}
We are now ready to complete the proof of \thmref{thm:SAOL}. 
\begin{proof} \textbf{(of \thmref{thm:SAOL})}
Consider an arbitrary interval $I=[q,s] \subseteq [T]$, and let $I=\bigcupdot_{i=-k}^ p I_i$ be the partition described in \lemref{lem:intervalPartition}. Then,  
\begin{align}  \label{eq:partitionSAOL} 
\regret_{\textrm{SAOL}^{\cB}}(I) &\le \sum_{i \le 0} \regret_{\textrm{SAOL}^{\cB}} (I_i) \notag \\&
+ \sum_{i \ge 1} \regret_{\textrm{SAOL}^{\cB}} (I_i)~.
\end{align}
We next bound the first term in the the right-hand side of \eqref{eq:partitionSAOL}. According to  \lemref{lem:coveredIntervals}, we obtain that
\begin{align*} 
\sum_{i \le 0} \regret_{\textrm{SAOL}^{\cB}} (I_i) &\le C \sum_{i \le 0}   |I_i|^\alpha \\&
+ 5 \sum_{i \le 0} \log(s_i+1)|I_i|^{1/2} \\&
\le C \sum_{i \le 0}   |I_i|^\alpha \\&
+ 5  \log(s+1) \sum_{i \le 0} |I_i|^{1/2}~. 
\end{align*}
According to \lemref{lem:intervalPartition},
\begin{align*}
\sum_{i \le 0} |I_i|^\alpha &\le \sum_{i=0} ^\infty (2^{-i}|I|)^{\alpha}\\&
= \frac{2^\alpha}{2^\alpha-1} |I|^{\alpha} \\& 
\le \frac{2}{2^{\alpha}-1} |I|^{\alpha}~.
\end{align*}
Similarly, we have
\begin{align*}
\sum_{i \le 0} |I_i|^{1/2} &\le \frac{\sqrt{2}}{\sqrt{2}-1} |I|^{1/2} \le  4 |I|^{\frac{1}{2}}~.
\end{align*}
Combining the three last inequalities, we obtain that
\begin{align*}
\sum_{i \le 0} \regret_{\textrm{SAOL}^{\cB}} (I_i) \le \frac{2}{2^\alpha-1} C |I|^{\alpha} + 20 \log(s+1) |I|^{\frac{1}{2}}~.
\end{align*}
The second term of the right-hand side of \eqref{eq:partitionSAOL} is bounded identically. Hence,
\[
\regret_{\textrm{SAOL}^{\cB}}(I) \le  \frac{4}{2^{\alpha}-1} C |I|^{\alpha} + 40 \log(s+1) |I|^{\frac{1}{2}}~.
\]
\end{proof}

\begin{figure}
\label{fig:cover}
\begin{center}
\begin{tikzpicture}
\draw (0.25,0) -- (7.5,0);
\node [below] at (0.25,-0.05) {$1$};
\node [below] at (0.5,-0.05) {$2$};
\node [below] at (1,-0.05) {$4$};
\node [below] at (1.75,-0.05) {$7$};
\node [below] at (7.5,-0.05) {$30$};
\node [below] at (2,-0.05) {$8$};
\node [below] at (4,-0.05) {$16$};
\node [below] at (6,-0.05) {$24$};
\node [below] at (7,-0.05) {$28$};

\foreach \x in {0.25,0.5,...,7.5}
{
\draw[fill] (\x,0) circle [radius=0.025];
}

% 1
\draw (0.25-0.05,-0.1)--(0.25-0.1,-0.1)--(0.25-0.1,0.1)--(0.25-0.05,0.1);
\draw (0.25+0.05,-0.1)--(0.25+0.1,-0.1)--(0.25+0.1,0.1)--(0.25+0.05,0.1);

% 2-3
\draw (0.5+0.05,-0.1)--(0.5,-0.1)--(0.5,0.1)--(0.5+0.05,0.1);
\draw (0.75-0.05,-0.1)--(0.75,-0.1)--(0.75,0.1)--(0.75-0.05,0.1);

% 4-7
\draw (1+0.05,-0.1)--(1,-0.1)--(1,0.1)--(1+0.05,0.1);
\draw (1.75-0.05,-0.1)--(1.75,-0.1)--(1.75,0.1)--(1.75-0.05,0.1);

% 8-15
\draw (2+0.05,-0.1)--(2,-0.1)--(2,0.1)--(2+0.05,0.1);
\draw (3.75-0.05,-0.1)--(3.75,-0.1)--(3.75,0.1)--(3.75-0.05,0.1);

% 16-23
\draw (4+0.05,-0.1)--(4,-0.1)--(4,0.1)--(4+0.05,0.1);
\draw (5.75-0.05,-0.1)--(5.75,-0.1)--(5.75,0.1)--(5.75-0.05,0.1);

% 24-27
\draw (6+0.05,-0.1)--(6,-0.1)--(6,0.1)--(6+0.05,0.1);
\draw (6.75-0.05,-0.1)--(6.75,-0.1)--(6.75,0.1)--(6.75-0.05,0.1);

% 28-29
\draw (7+0.05,-0.1)--(7,-0.1)--(7,0.1)--(7+0.05,0.1);
\draw (7.25-0.05,-0.1)--(7.25,-0.1)--(7.25,0.1)--(7.25-0.05,0.1);

% 30
\draw (7.5-0.05,-0.1)--(7.5-0.1,-0.1)--(7.5-0.1,0.1)--(7.5-0.05,0.1);
\draw (7.5+0.05,-0.1)--(7.5+0.1,-0.1)--(7.5+0.1,0.1)--(7.5+0.05,0.1);

% \foreach \x in {4} 
% { 
% \draw (\x+0.05,-0.1)--(\x,-0.1)--(\x,0.1)--(\x+0.05,0.1);
% \draw (\x+3.45,-0.1)--(\x+ 3.5,-0.1)--(\x+3.5,0.1)--(\x+3.45,0.1);
% }

\end{tikzpicture}
\end{center}
\caption{Geometric Covering of Interval: The interval $I=[1,30]$ is partitioned into the sequences $(I_{-1} = [1] ,I_{-2} = [2,3] ,I_{-1}= [4,7],I_0=[8,15])$ and $(I_1 = [16,23], I_2 =[24,27], I_3 = [28,29], I_4=[30])$}
\end{figure}
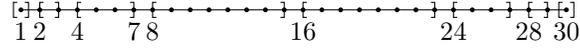

\end{document}